\documentclass[12pt, reqno]{amsart}
\usepackage{algorithm}
\usepackage{algpseudocode}
\usepackage[left=1.1in,top=1.2in,right=1.1in,bottom=1.2in]{geometry}
\usepackage{setspace,kantlipsum}
\usepackage{amsfonts}
\usepackage{amssymb}
\usepackage{amsmath}
\usepackage{amsthm}
\usepackage{latexsym}
\usepackage{graphicx}
\usepackage{xy}
\usepackage{hyperref}
\usepackage{lineno}
\usepackage{tikz-cd} 
\usepackage{subcaption} 
\usepackage{float}
\usepackage{booktabs}
\usepackage{array}

 \usepackage{listings}
\usepackage[bottom]{footmisc}
\renewcommand{\thefootnote}{\fnsymbol{footnote}}
\newcommand{\ui}[1]{{\left\vert\kern-0.25ex\left\vert\kern-0.25ex\left\vert #1 
    \right\vert\kern-0.25ex\right\vert\kern-0.25ex\right\vert}}
\makeatletter
\@namedef{subjclassname@2020}{%
  \textup{2020} Mathematics Subject Classification}
\makeatother

\usepackage{doc}
\makeatletter
\def\maketitle{\par
      \begingroup
        \def\thefootnote{\fnsymbol{footnote}}%
      \setcounter{footnote}\z@
      \def\@makefnmark{\hbox to\z@{$\m@th^{\@thefnmark}$\hss}}%
      \long\def\@makefntext##1{\noindent
          \ifnum\c@footnote>\z@\relax
            \hbox to1.8em{\hss$\m@th^{\@thefnmark}$}##1%
          \else
          \hbox to1.8em{\hfill}%
            \parbox{\dimexpr\linewidth-1.8em}{\raggedright ##1}%
          \fi}
      \if@twocolumn\twocolumn[\@maketitle]%
      \else\newpage\global\@topnum\z@\@maketitle\fi
      \thispagestyle{titlepage}\@thanks\endgroup
      \setcounter{footnote}\z@
      \gdef\@date{\today}\gdef\@thanks{}%
      \gdef\@author{}\gdef\@title{}\gdef\@dedicatory{}}

\newcommand{\brown}[1]{\textcolor[rgb]{0.60,0.30,0.10}{#1}}    
\hypersetup{colorlinks,citecolor=red,pdfstartview=FitH, pdfpagemode=None}

\usepackage{xcolor, soul}

\numberwithin{equation}{section}

\theoremstyle{definition}
\newtheorem{theorem}{Theorem}[section]
\newtheorem{lemma}[theorem]{Lemma}
\newtheorem{proposition}[theorem]{Proposition}

\newtheorem{definition}[theorem]{Definition}

\newtheorem{corollary}[theorem]{Corollary}

\newtheorem{remark}[theorem]{Remark}

\font\germ=eufm10

\def\a{\mbox{\germ a}}

\def\p{\mbox{\germ p}}

\def\s{\mbox{\germ s}}

\def\u{\mbox{\germ u}}
\def\z{\mathbf z}

\def\ui{\|\hspace{-.25mm} |\,}

\def\R{\mathbb R}
\def\C{\mathbb C}

\def\Cnn{{\mathbb C}_{n\times n}}

\def\F{\mathbb F}

\def\H{\mathbb H}

\def\A{\mathcal A}
\def\a{\textbf{\em a}}

\def\fa{\mathfrak a}

\def\diag{{\mbox{diag}\,}}

\def\rank{\mbox{\rm rank\,}}

\def\ad{{\operatorname{ad}\,}}
\def\Ad{{\operatorname{Ad}\,}}

\def\Gr{{\mbox{\bf Gr}}}

\begin{document}

\title[Quaternionic Grassmannians]{ Color Image Set Recognition Based on Quaternionic Grassmannians}

\author{Xiang Xiang Wang}
\address{Department of Mathematics and Statistics\\ University of Nevada, Reno\\ Reno \\ NV 89557-0084\\ USA}
\email{xiangxiangw@unr.edu}

\author{Tin-Yau Tam}
\address{Department of Mathematics and Statistics\\ University of Nevada, Reno\\ Reno \\ NV 89557-0084\\ USA}
\email{ttam@unr.edu}
%


\begin{abstract}  
We propose a new method for recognizing color image sets using quaternionic Grassmannians, which use the power of quaternions to capture color information and represent each color image set as a point on the quaternionic Grassmannian. We provide a direct formula to calculate the shortest distance between two points on the quaternionic Grassmannian, and use this distance to build a new classification framework. Experiments on the ETH-80 benchmark dataset and the Highway Traffic video dataset show that our method achieves good recognition results. We also discuss some limitations in stability and suggest ways the method can be improved in the future.

\end{abstract}

\maketitle


\section{Introduction}
Recognizing image sets is an important task in computer vision, with applications in areas such as face recognition, object tracking, and video analysis. Instead of processing individual images independently, image set recognition methods treat a set of images collectively, leading to more robust performance under varying conditions such as lighting, pose, and occlusion \cite{kim2007, hamm2008}. This paradigm is particularly useful in real-world applications where each subject or object is captured under multiple views or conditions, as seen in surveillance systems, video-based biometric authentication, and behavior recognition \cite{Har2011, liu2003video}.

A Grassmannian manifold is the collection of all subspaces of a fixed dimension $k$ within an $n$-dimensional real or complex space. It offers a natural and effective mathematical framework for comparing image sets based on the geometry of subspaces. By modeling each image set as a point on the Grassmannian, these methods leverage intrinsic manifold structure to compute meaningful distances and enable classification using methods informed by the underlying geometry.  
 
 In recent years, various methods have been developed using Grassmannians for image set recognition, where image sets are represented as subspaces in higher-dimensional spaces \cite{souza2023, wang2011}. 
 One of the earliest and most fundamental approaches is the Mutual Subspace Method (MSM) proposed by Yamaguchi, Fukui, and Maeda in 1998 \cite{yamaguchi1998face}, which compares subspaces using principal angles. Building on this idea, the Grassmannian Nearest Neighbor (GNN) classifier was widely adopted as a baseline method in later studies, using distances such as projection or geodesic distance between subspaces for classification \cite{hamm2008}.
In 2008, Hamm and Lee introduced Grassmann Discriminant Analysis (GDA), which performs dimensionality reduction on the Grassmann manifold to enhance discrimination between classes \cite{hamm2008}. This was extended in 2011 by Harandi et al. through the Grassmannian Graph Embedding Discriminant Analysis (GEDA), which incorporated graph embedding techniques for improved class separation \cite{Har2011}. Later, in 2013, Harandi and colleagues proposed Grassmannian Discriminant Learning (GDL), which formulated a discriminative framework based on manifold geometry for more effective subspace learning \cite{Har2013}.
To further enhance local structure preservation, Grassmannian Locality Preserving Projection (GLPP) was introduced by Kumar in 2019, and also explored by Wang et al., as a method that maintains neighborhood relationships within the Grassmannian structure \cite{kumar2019jumping}. More recently, Wei, Shen, Sun, Gao, and Ren proposed several new models. In 2022, they developed the Grassmannian Neighborhood Preserving Embedding (GNPE) method, which focuses on preserving the neighborhood structure during embedding \cite{wei2022neighborhood}. In 2024, the same group introduced Grassmannian Adaptive Local Learning (GALL) and its variant F-norm based Grassmannian Adaptive Local Learning (F-GALL), which were derived from adaptive optimization formulations tailored to local learning on the  Grassmannian \cite{wei2024}.
Together, these methods represent the evolution of Grassmannian-based approaches in image set classification, progressively incorporating more sophisticated structures such as local geometry, discriminant embedding, and adaptive learning mechanisms.

 Although these methods have been successful, many current techniques for handling color images still rely on traditional methods that treat the RGB channels separately. While this can work in certain situations, it misses the connections between the color channels, making it harder to capture the full structure of the color images. As a result, valuable inter-channel correlations are lost, which can lead to suboptimal performance, especially in tasks where color is a crucial distinguishing factor.

Quaternions, which extend complex numbers into four dimensions, offer a useful approach by allowing all three RGB channels to be stored together in one quaternionic matrix. This provides a compact way to store the data while keeping the relationships between the color channels, making it a more efficient and meaningful representation for color images \cite{le2003, pei1999}. 
Quaternion-based representations have been successfully applied in color image filtering, edge detection, and recognition, as they preserve chromatic information and allow for algebraically elegant operations in multidimensional color space. 
By embedding color images into quaternionic Grassmannians, we can take advantage of both quaternions and Grassmannians, creating a powerful method for recognizing image sets.

One of the key challenges in this area is figuring out how to measure distances between points in quaternionic Grassmannian space. These distance measurements are important for comparing image sets and performing classification. Finding a way to reliably and efficiently calculate the distance between two points in quaternionic Grassmannians has been a long-standing problem. To address this, we provide a clear mathematical expression for the shortest geodesic distance between two points in quaternionic Grassmannian space, using matrices. This distance formula is central to our new framework for recognizing color image sets. Our method bridges the gap between quaternion algebra and Riemannian geometry, offering a novel tool for structure-preserving image set analysis.


To evaluate the effectiveness of our method, we conduct experiments on benchmark dataset and compare our approach with several established Grassmannian-based methods for image set recognition, including  a range of approaches such as Grassmannian
Nearest Neighbor (GNN), Grassmann Discriminant Analysis
(GDA), Grassmannian Graph Embedding Discriminant Analysis (GEDA), Grassmannian Discriminant Learning (GDL), among others.  These comparisons highlight the advantages of our quaternionic framework in capturing both geometric and chromatic information.

Furthermore, our framework naturally extends to color video recognition, where each video can be treated as a set of frames represented by quaternionic subspaces. This extension demonstrates the flexibility and generalizability of our approach in dynamic scenarios.

The rest of this paper is organized as follows: Section \ref{sec2} reviews background information and related work on the quaternion algebra for color image representation and the quaternionic unitary  group. In Section \ref{sec3}, we present the mathematical details of how to calculate the shortest distance in quaternionic Grassmannians. Section \ref{sec4} describes our proposed framework for recognizing color image sets. Section \ref{sec5} shows experimental results that demonstrate how well our method works. Finally, Section \ref{sec6} discusses future work.

\section{Background}\label{sec2}
In this section, we provide an overview of quaternion algebra for color image representation and quaternionic unitary group, which together form the mathematical foundation of our method for representing and analyzing color image sets using quaternionic Grassmannians. We adhere to the standard notations listed in Table~\ref{tab:notations} to ensure clarity and consistency. These notations will be used throughout the paper.
\begin{table}[h]
   \centering
   \begin{tabular}{c|l}
      \toprule
      \textbf{Symbol}   & \textbf{Name} \\ 
      \midrule
      $I$&The identity matrix with the size $n\times n$\\
            $\R$ & The set of all real numbers\\
      $\C$ & The set of all complex numbers\\
         $\H$ & The set of all quaternion numbers\\
            $\F$ & Represents either $\R$, $\C$ or $\H$\\
    $\F_{n\times m}$ & The set of all   matrices with  the size $n\times m$ in $\F$\\
    $\F_{n}$ or $\F_{n\times 1}$ & The set of all  vectors with  the size $n$ in $\F$\\
        $M$ & A smooth manifold\\
    $T_x M$ &The tangent space of $M$ at point $x \in M$\\
    $U(n)$& Unitary group\\
    $\s_{\H}(n)$ & The  space of $n\times n$ skew  quaternionic Hermitian matrices\\
      $U_{\H}(n)$ &Quaternionic unitary group\\
      $\Gr_{n,k}(\F)$ &   Grassmannian with $k$-dimensional subspaces in $\F_n$\\
      $S_n$ & The space of $n\times n$ real symmetric matrices\\
      $H_n$ &The  space of $n\times n$ Hermitian matrices\\
      $Q_n$ & The space of $n\times n$ quaternionic Hermitian matrices\\
      $\|\cdot\|_F$ & Frobenius norm\\
       $\|\cdot\|_{\H}$ & Frobenius norm for quaternionic matrix\\
            $\u(n)$ & The set of $n\times n$ skew Hermitian matrices\\
     $\exp(\cdot)$ or $e^{\cdot}$ & Exponential map of the matrix\\
     $\chi_H$ &The complex representation
of the quaternionic matrix $H$\\
$\cdot^*$ & The transpose and conjugate notation \\
$\bar{\cdot}$ & The conjugate notation \\
$ \sinh(M) $& The hyperbolic sine of the matrix $M$: $\sinh(M)= \frac{e^M -e^{-M}}{2}$\\ 
$[A,B]$& Lie bracket:  $[A,B]=AB-BA$\\
         \bottomrule
   \end{tabular}
   \caption{Notations used in this paper}   \label{tab:notations}
\end{table}

\subsection{Quaternion Algebra for Color Image Representation}

Quaternions, introduced by Hamilton in 1843 \cite{hamilton1844quaternions}, extend complex numbers to four dimensions. They are widely used in computer graphics \cite{shoemake1985animating,kuipers1999quaternions}, robotics \cite{altmann1986rotations}, and signal processing \cite{zhang1997quaternion}. In the context of image processing, quaternions provide a natural and compact way to encode color images by combining the RGB channels into a single quaternion-valued matrix \cite{le2003, pei1999}.

A quaternion $q \in \H$ can be written as:
\[
q = q_0 + q_1 i + q_2 j + q_3 k,
\]
where $q_0, q_1, q_2, q_3 \in \R$, and \( i, j, \) and \( k \) are the basis elements satisfying the fundamental relations:
\[
i^2 = j^2 = k^2 = ijk = -1.
\]

The conjugate and norm of $q$ are defined respectively by:
\[\bar{q} = q_0 - q_1 i - q_2 j - q_3 k, \quad |q| = \sqrt{q\bar{q}} = \sqrt{q_0^2 + q_1^2 + q_2^2 + q_3^2}.
\]

For any nonzero quaternion $q$, the inverse is given by:
\[q^{-1} = \frac{\bar{q}}{|q|^2}.
\]

A quaternionic matrix $H \in \H_{n\times m}$ has the form:
\[H = H_0 + H_1 i + H_2 j + H_3 k,
\]
where $H_0, H_1, H_2, H_3 \in \R_{n\times m}$. This can also be written as:
\[H = (H_0 + H_1 i) + (H_2 + H_3 i) j,
\]
which is which is a convenient form for deriving its complex representation.
The complex representation of the quaternionic matrix \( H \), denoted as \( \chi_H \), is defined by:\begin{equation}\label{complexRe}
\chi_H = \begin{bmatrix}
H_0 + H_1 i & H_2 + H_3 i \\
-\overline{H_2 + H_3 i} & \overline{H_0 + H_1 i}
\end{bmatrix} =
\begin{bmatrix}
H_0 + H_1 i & H_2 + H_3 i \\
- H_2 + H_3 i & H_0 - H_1 i
\end{bmatrix}.
\end{equation}

This complex representation \eqref{complexRe} preserves   several important properties of quaternionic matrices, as outlined in the following proposition.

\begin{proposition} \label{Complex} [Lee, 1948 \cite{lee1948}]
Let $A,B \in \H_{n\times n}$. Then:
\begin{itemize}
  \item $\chi_{AB} = \chi_A \chi_B$;
  \item $\chi_{A+B} = \chi_A + \chi_B$;
  \item $\chi_{A^*} = (\chi_A)^*$;
  \item If $A^{-1}$ exists, then $\chi_{A^{-1}} = (\chi_A)^{-1}$;
  \item $\chi_A$ is unitary, Hermitian, or normal if and only if $A$ is unitary, Hermitian, or normal, respectively.
\end{itemize}
\end{proposition}

Moving forward, we introduce the concept of ``standard eigenvalues'' for quaternionic matrices, which are crucial for understanding their spectral properties and have significant applications.
\begin{definition}[Brenner, 1951; Lee, 1948 \cite{Brenner1951, lee1948}]
For any \( n \times n \) quaternionic matrix \( A \), there exist exactly \( n \) (right) eigenvalues that are complex numbers with non-negative imaginary parts. These eigenvalues are referred to as the \emph{standard eigenvalues} of \( A \).
\end{definition}

The following lemma offers a good understanding of these eigenvalues and highlights key properties of the eigenvalue structure for related complex matrices.
\begin{lemma}[Lee, 1949 \cite{lee1948}]
Let \( A \) and \( B \) be \( n \times n \) complex matrices. Then, for the block matrix:
\[
\begin{pmatrix}
A & B \\
- \bar{B} & \bar{A}
\end{pmatrix},
\]
every real eigenvalue (if any) appears an even number of times, while the complex eigenvalues occur in conjugate pairs.
\end{lemma}

This observation leads to the following corollary regarding the eigenvalues of quaternionic
matrices.
\begin{corollary}
Let \( H \in \mathbb{H}_{n \times n} \) be a quaternionic matrix, and let \( \chi_H \) denote its complex representation. Then \( \chi_H \) has exactly \( 2n \) complex eigenvalues, which are symmetrically distributed with respect to the real axis in the complex plane.
Among these, exactly \( n \) eigenvalues lie in the closed upper half-plane (i.e., the set of complex numbers with non-negative imaginary part). These \( n \) eigenvalues are referred to as the \emph{standard eigenvalues} of \( H \), and they correspond to the eigenvalues of \( \chi_H \) located in the upper half-plane.
\end{corollary}
Further details can be found in the proof of Theorem 5.4 in~\cite{zhang97}. These spectral insights not only guide theoretical understanding but also motivate how we interpret eigenvalues in quaternionic settings.

To formalize relationships between quaternions, we introduce the notion of similarity, which plays an important role in the classification of eigenvalues.
\begin{definition}
Two quaternions \( q \) and \( p \) are said to be similar if there exists a nonzero quaternion \( s \) such that:
\[
p = s^{-1} q s.
\]
\end{definition}
This definition leads to an important property: similar quaternions preserve their norms, since similarity is a type of isometric transformation.

\begin{remark}
In the context of quaternionic matrices, this concept of similarity helps explain the structure of right eigenvalues. Specifically, any right eigenvalue of a quaternionic matrix is similar to one of its standard eigenvalues. This is why standard eigenvalues are especially useful in analysis: once the \( n \) standard eigenvalues are obtained, they represent all possible right eigenvalues through similarity transformations. 
\end{remark}


The spectral properties of quaternionic matrices, especially their standard eigenvalues, are useful in many real-world applications. One important area where quaternion algebra shows its value is in color image processing.

In color image representation, each pixel of a color image can be represented by a pure quaternion
$q=0+q_1i+q_2j+q_3k$, where the components \( q_1 \), \( q_2 \), and \( q_3 \) correspond to the values of the red, green, and blue channels, respectively. Consequently, a color image of size \( n \times m \) can be represented as a pure \( n \times m \) quaternionic matrix (see Figure~\ref{ColorQ}), with each entry in the matrix being a pure quaternion that encodes the RGB information of the corresponding pixel.

\begin{figure}[htbp]
\begin{center}
 \includegraphics[width=3.5in]{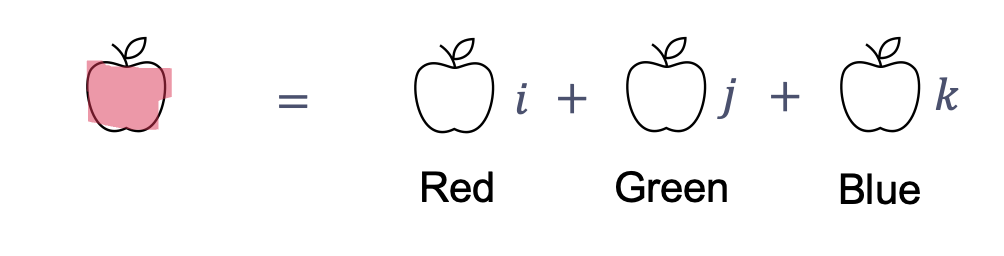} 
\caption{A color image represented as a pure quaternionic matrix}
\label{ColorQ}
\end{center}
\end{figure}

This quaternion-based representation enables compact and unified processing of the color channels while preserving the relationships between them. The key idea is that the quaternionic structure inherently encodes the correlations between the RGB channels, as opposed to treating them as three separate grayscale images. This inter-channel dependency is especially beneficial when applying linear algebraic operations or learning-based models, as it preserves perceptually meaningful structure and improves robustness against variations in lighting and noise \cite{pei1999, le2003}. Moreover, quaternionic coefficients preserve color phase information and encode global relationships between channels, which can lead to improved invariance under transformations such as rotation, illumination changes, and channel permutation \cite{4032812, zhang1997quaternion}.

Furthermore, in the context of image set recognition, these quaternionic representations allow us to model sets of color images as subspaces in a quaternionic vector space. This facilitates the construction of quaternionic Grassmannians, where each image set corresponds to a point on the manifold. Such a geometric framework supports the definition of meaningful distances between sets while preserving the color structure and inter-channel relationships.

However, most existing approaches using quaternions focus on pixel-wise operations or per-image tasks, rather than extending the representation to sets of images. This motivates the use of quaternionic Grassmannians, where image sets are treated as subspaces in a quaternionic vector space, offering a unified representation that is both perceptually and geometrically informed.



\subsection{Quaternionic Unitary Group}

The quaternionic unitary group \( U_{\H}(n) \) plays a central role in defining the geometry of quaternionic Grassmannians. It is also commonly known as the \textit{compact symplectic group}, denoted by \( \mathrm{Sp}(n) \) \cite{djokovic1979products}, which should not be confused with the real symplectic group \( \mathrm{Sp}(2n, \mathbb{R}) \). The group \( U_{\H}(n) \) consists of all \( n \times n \) quaternionic matrices \( Q \) satisfying the unitary condition:
\[
Q^* Q = I.
\]

To understand the geometry of this group, we next describe its tangent space, geodesics, and real dimension.
\subsubsection{Tangent Space}
The tangent space of \( U_{\H}(n) \) at the identity matrix \( I \) is the space of skew-Hermitian quaternionic matrices:
\[
\s_{\H}(n) = \{ H \in \H_{n\times n} : H^* = -H \}.
\]
This can be derived using a differentiable curve \( \gamma(t) \in U_{\H}(n) \) such that \( \gamma(0) = I \). The unitary condition \( \gamma(t)^* \gamma(t) = I \) implies via differentiation at \( t = 0 \) that \( S = \dot{\gamma}(0) \) satisfies \( S^* = -S \) \cite{XU2015}.

Conversely, for any \( S \in \s_{\H}(n) \), the curve \( \gamma(t) = \exp(tS) \) remains entirely within \( U_{\H}(n) \), as the exponential of a skew-Hermitian matrix is unitary. Thus,
\[
T_I U_{\H}(n) = \s_{\H}(n).
\]

At an arbitrary point \( Q \in U_{\H}(n) \), the tangent space is given by left translation:
\[
T_Q U_{\H}(n) = \{ QX : X \in \s_{\H}(n) \}.
\]

\subsubsection{Geodesics}
Since \( U_{\H}(n) \) is a Lie group with a bi-invariant Riemannian metric, the geodesics are one-parameter subgroups:
\[
\gamma(t) = Q \exp(tX), \quad X \in \s_{\H}(n).
\]

\subsubsection{Dimension}
The real dimension of \( U_{\H}(n) \) is equal to the real dimension of its Lie algebra \( \s_{\H}(n) \). A skew-Hermitian quaternionic matrix has \( n \) purely imaginary quaternion entries on the diagonal (each contributing 3 real parameters), and \( \frac{n(n-1)}{2} \) arbitrary quaternion entries above the diagonal (each contributing 4 real parameters), because the skew-Hermitian condition determines the entries below the diagonal. Hence,
\begin{equation}\label{dimension}
\dim U_{\H}(n) = 3n + 4 \cdot \frac{n(n-1)}{2} = n(2n + 1).
\end{equation}
This result is consistent with known properties of the compact symplectic group \( \mathrm{Sp}(n) \), which is diffeomorphic to \( U_{\H}(n) \) \cite{djokovic1979products}.

This geometric framework for \( U_{\H}(n) \) serves as the foundation for defining and analyzing quaternionic Grassmannians, which we discuss in the next section.


\section{ Quaternionic Grassmannian} \label{sec3}
Quaternionic Grassmannian $\Gr_{n,k }(\H)$ is the set of all $k$-dimensional subspaces in $\H_n$ and 
we can define the quaternionic Grassmannian as a set of quaternionic matrices by orthogonal projection matrices in $\H_{n\times n}$:
\begin{equation} \Gr_{n,k }(\H)=\{ P\in \H_{n\times n}: P^2=P,P^*=P,\ \rank P=k\},\end{equation}
which is equivalent to the quotient space
\[U_{\H}(n)/(U_{\H}(k)\times U_{\H}(n-k)),
\]

where $U_{\H}(n)=\{U\in \H_{n\times n}: U^*U=UU^*=I_n\}$ is the quaternionic unitary group.

\subsection{Tangent space and geodesic of $\Gr_{n,k }(\H)$}
We now derive the tangent space and the geodesic form of $\Gr_{n,k }(\H)$ based on $U_{\H}(n)$.
Consider a curve  starting from a point $P \in \Gr_{n,k }(\H)$,   parameterized as $\alpha(t)=\beta(t)P\beta(t)^*$ with $\alpha(0)=P$, 
where
$\beta(t)\in U_{n}(\H)$ and $\beta(0)=I$.

\subsubsection{Tangent Space} Since  the tangent space of  $U_{\H}(n)$ at $I$ is $\s_{\H}(n)$, we have
$\dot{\beta}(0)=X\in \s_{\H}(n)$ and 
the tangent vector is $\dot{\alpha}(0)=[X,P]$. Therefore, the tangent space at 
$P\in \Gr_{n,k }(\H)$ is
\begin{equation}
T_P\Gr_{n,k }(\H)=\{[X,P]: X\in \s_{\H}(n)\}.\end{equation}
 To emphasize the roles of $P$ and $Q$ in the tangent vector, we write
respectively the tangent vector of the geodesic from $P$ to $Q$ and the
tangent vector from $Q$ to $P$ as
\[\overrightarrow{PQ} =[X_{PQ},P], \quad \overrightarrow{QP}=[X_{QP},Q].\]

Since each $P\in \Gr_{n,k }(\H)$ is a Hermitian quaternionic matrix with eigenvalues  consisting of  $k$ ones and 
$n-k$  zeros, the spectral decomposition of $P=U_P P_0U_P^*$ yields the following map:
 \[ \pi: U_{\H}(n) \to \Gr_{n,k} (\H), U\to UP_0 U^*,
 \]
 which is surjective.

A similar approach to the complex Grassmannian case in \cite{BT2020} allows us to express the tangent space of the quaternionic Grassmannian as follows 
\begin{equation}\label{XPUpstar} T_P \Gr_{n,k }(\H)= \{[X,P]:X\in \Ad(U_P)\p_*\},
\end{equation}
where 
\begin{equation}\label{pstar}\p_*: \left\{\begin{pmatrix} 0&Y_1\\-Y^*_1&0
\end{pmatrix} :Y_1\in \H_{k\times (n-k)}\right\}\subset  \s_{\H}(n),\end{equation}

and $\Ad(U)X=UXU^{-1}.$

\subsubsection{Geodesic} Furthermore, the geodesic starting from $P$  in the direction $[X,P]$, where  $X=U_P\hat{X}U_p$ and $\hat{X}\in \p_*$  in $\Gr_{n,k}(\H)$ can be derived from the corresponding curve in $U_{\H}(n)$ with the form $U_P \hat{X}$.
Thus, the geodesic of the Grassmannian is

\begin{equation} \label{gammaXP}
\gamma(t)=\exp(tX)P\exp(-tX). \end{equation}
 

\subsubsection{Dimension}
To verify that the set of tangent vectors \( \{[X, P] : X \in \s_{\H}(n)\} \) spans the full tangent space \( T_P \Gr_{n,k}(\H) \), we perform a dimension count based on differential geometric principles and Lie group theory. We begin by recalling a fundamental result from differential geometry:
\begin{proposition}[Proposition 3.12 in \cite{lee2012smooth}]
Let \( M \) be an \( n \)-dimensional smooth manifold with boundary. Then for each point \( p \in M \), the tangent space \( T_p M \) is an \( n \)-dimensional real vector space.
\end{proposition}
This implies that to demonstrate \( \{[X, P] : X \in \s_{\H}(n)\} \) is indeed the full tangent space, it suffices to show that it has the same real dimension as the Grassmannian \( \Gr_{n,k}(\H) \) itself.

To this end, we compute the dimension of \( \Gr_{n,k}(\H) \) using its homogeneous space structure. The quaternionic Grassmannian can be written as symmetric space:
\[
\Gr_{n,k}(\H) \cong \brown{U_{\H}(n) / (U_{\H}(k) \times U_{\H}(n-k))},
\]
where \( U_{\H}(n) \) denotes the quaternionic unitary group. According to \cite{djokovic1979products} and \eqref{dimension}, the real dimension of \( U_{\H}(n) \) is \( n(2n + 1) \). Using the standard formula for the dimension of a homogeneous space, we have:
\[\begin{split}
\dim \Gr_{n,k}(\H) &= \dim U_{\H}(n) - \dim U_{\H}(k) - \dim U_{\H}(n-k)\\
& = n(2n + 1) - k(2k + 1) - (n - k)(2(n - k) + 1).
\end{split}
\]
Simplifying the above expression yields:
\[
\dim \Gr_{n,k}(\H) = 4k(n - k).
\]
Next, from our formulation in Equations~\eqref{XPUpstar} and~\eqref{pstar}, we define a candidate tangent space at \( P \in \Gr_{n,k}(\H) \) as:
\[
\{[X, P] : X \in \s_{\H}(n)\}  \cong \{[X,P]:X\in \Ad(U_P)\p_*\},
\]
and 
\[ \dim \{[X, P] : X \in \s_{\H}(n)\}  =\dim \{Y_1: Y_1\in \H_{k\times(n-k)}\}=4k(n-k).\]
Therefore, \[
\dim T_P \Gr_{n,k}(\H) = \dim \{[X, P] : X \in \s_{\H}(n)\} = 4k(n - k),
\]
and we conclude that \( \{[X, P] : X \in \s_{\H}(n)\} \) indeed spans the tangent space \( T_P \Gr_{n,k}(\H) \).


\subsection{The Shortest Distance Between Two Points in \( \Gr_{n,k }(\H) \)}

To study the geometry of quaternionic Grassmannians \( \Gr_{n,k}(\H) \), we begin by characterizing geodesics between two points and defining their shortest distances.

\begin{remark}
It is important to note that geodesics on the Grassmannian manifold are not necessarily unique. There may exist multiple geodesics connecting two given points \( P \) and \( Q \), each corresponding to a different length. Therefore, when we refer to the ``shortest distance'' on the Grassmannian, we specifically mean the minimal geodesic distance, that is, the length of the shortest path among all possible geodesics connecting the two points. 
\end{remark}

Let \( P \in \Gr_{n,k}(\H) \) and consider a tangent vector \( [X, P] \in T_P \Gr_{n,k}(\H) \). The following properties hold for such tangent vectors:
\begin{enumerate}
    \item \( X = PX + XP \);
    \item \( [X, P] = (I - 2P)X = -X(I - 2P) \);
    \item \( \exp(X)P - P\exp(-X) = \sinh M \).
\end{enumerate}
These results are derived from Lemma 2.1 and Lemma 3.2 in \cite{BK2015}. They are essential in understanding how geodesics evolve on the Grassmannian.

Furthermore,  following Theorem 3.3 in \cite{BK2015}, for any two points \( P, Q \in \Gr_{n,k }(\H) \), there exists a geodesic curve connecting them, given by
\begin{equation}\label{gammaXP}
\gamma(t) = \exp(tX)P\exp(-tX), \quad t \in [0,1],
\end{equation}
where the skew-Hermitian quaternionic matrix \( X \in \s_{\H}(n) \) satisfies the boundary condition
\[
\exp(2X) = (I - 2Q)(I - 2P).
\]
This form ensures that \( \gamma(0) = P \) and \( \gamma(1) = Q \), making \( \gamma(t) \) a valid geodesic on the Grassmannian manifold.

To further analyze distances on \( \Gr_{n,k}(\H) \), we adopt the quaternionic Schatten-2 norm, as introduced in \cite{MiaoKou2020}. This norm is defined based on the singular value decomposition (SVD) of quaternionic matrices, which was rigorously established in \cite{zhang97}.
\begin{theorem}[Singular-Value Decomposition \cite{zhang97}]
Let \( A \in \H_{n \times m} \) be a quaternionic matrix of rank \( r \). Then there exist quaternionic unitary matrices \( U \in \H_{n \times n} \) and \( V \in \H_{m \times m} \) such that
\[
UAV = \begin{pmatrix} D_r & 0 \\ 0 & 0 \end{pmatrix},
\]
where \( D_r = \diag(d_1, \dots, d_r) \) and the \( d_i \)'s are the positive singular values of \( A \).
\end{theorem}

The singular values from the SVD are then used to define the quaternionic Schatten-2 norm for quaternionic matrices. Specifically,   the quaternionic Schatten-2 norm of quaternionic matrix $A = (a_{ij}) \in \H_{n \times m} $ is defined by
\begin{equation}
\| A \|_{\H} = \sqrt{\sum_{i,j} |a_{ij}|^2} = \sqrt{\sum_{k=1}^{r} \sigma_k^2(A)},
\end{equation}
where \( \sigma_k(A) \) are the nonzero singular values of \( A \).

With this norm in place, we can now define the distance between two points \( P, Q \in \Gr_{n,k}(\H) \). The distance is given by
\begin{equation}
d(P, Q) = \| [X, P] \|_{\H} = \| X \|_{\H},
\end{equation}
where \( X \) is the matrix associated with the geodesic connecting \( P \) and \( Q \).

This leads us to the following theorem, which provides an explicit expression for the shortest distance between any two points in \( \Gr_{n,k}(\H) \).

 Let $X\in U_{\H}(n)$, and let $\hat{\lambda}(X)=(\hat{\lambda}_1(X),...,\hat{\lambda}_n(X))$ be the set of standard eigenvalues of $X$.
For any complex number $c$ on the unit circle, we can express $c =e^{i\alpha}$ with $\alpha\in [-\pi,\pi]$. Denote $\overline{\arg}(c)=\alpha$ as the argument of $c$. In particular, for $c=-1$, we can take either $\pi$ or $-\pi$ as its argument, that is, $\overline{\arg}(-1)=\pi$ or   $-\pi$.

\begin{theorem}
Let $P$ and $Q$ be two points in $\Gr_{n,k }(\H)$. Then 
the shortest distance between $P$ and $Q$ is 
\begin{equation}\label{dF}
\hat{d}(P,Q)=\frac{1}{2}\sqrt{\sum_j \overline{\arg}^2 \big(\hat{\lambda}_j \big((I-2Q)(I-2P) \big) \big)}.
\end{equation}
 
\end{theorem}

\begin{proof}
Suppose that there are $t$ geodesics between $P$ and $Q$. The geodesic connecting $P$ and $Q$ with the tangent vector $[X_j, P]$ is given by:
\[
\gamma(t) = \exp(tX_j) P \exp(-tX_j), \quad \text{for } j = 1, 2, \dots, t.
\]
The distance between $P$ and $Q$ along the geodesic with the tangent vector $[X_j, P]$ at the point $P$ is:
\[
d_i(P, Q) = \|[X_j, P]\|_{\H} = \|X_j\|_{\H}.
\]
Thus, the shortest distance between $P$ and $Q$ is:
\[
\hat{d}(P, Q) = \min_j d_j(P, Q) = \min_j \|X_j\|_{\H}.
\]

For each $X_j$, we have the following relation:
\begin{equation}
\exp(2X_j) = (I - 2Q)(I - 2P).
\end{equation}
Since $(I - 2Q)(I - 2P)$ is a quaternionic unitary matrix, the eigenvalues $\hat{\lambda}_j\big( (I - 2Q)(I - 2P) \big)$ are complex numbers on the unit circle, and the remaining right eigenvalues are similar to the standard eigenvalues.
Let us write the eigenvalues as:
\[
\hat{\lambda}_j \big( (I - 2Q)(I - 2P) \big) = \exp(i a_j),
\]
where $a_j \in [-\pi, \pi]$. Since similar eigenvalues have the same norm, we obtain:

\[
\min_j \|2X_j\|_{\H} = \sqrt{\sum_j a_j^2}.
\]

Therefore, we have
\[
\hat{d}(P, Q) = \frac{1}{2} \min_j \|2X_j\|_{\H} = \frac{1}{2} \sqrt{\sum_j \overline{\arg}^2 \big( \hat{\lambda}_j \big( (I - 2Q)(I - 2P) \big) \big)}.
\]

\end{proof}

\begin{remark}
The distance formula \eqref{dF}, which is also compatible with real and complex Grassmannians,    is a well-defined metric as it satisfies the three fundamental conditions of a metric space: ``non-negativity'', ``symmetry'', and the ``triangle inequality''. A detailed proof is provided below.
 
\end{remark}
\subsection{Proof of the Triangle Inequality for the Distance Formula \eqref{dF}}
It is straightforward to verify that the non-negativity and symmetry conditions hold for the distance formula. To complete the proof that it is a well-defined metric, we need to establish the triangle inequality. To do this, we first introduce some key properties of the complex Grassmannian $\Gr_{n,k}(\C)$, which will help guide our approach.
We consider the model of Grassmannians as the set of orthogonal projection matrices in $\Cnn$ of $\rank k$: 
$$
\Gr_{n,k}(\C)=\{P\in \Cnn: P^2 =P, \, P^*=P,\ \rank P = k\},
$$
which is the subset of the space of $n\times n$ Hermitian matrices.

The properties of \( \Gr_{n,k}(\C) \) have been well studied, particularly in \cite{BT2024} and \cite{BK2015}, and we will use some of their results in our proof.

First, the tangent space of any point $P\in \Gr_{n,k}(\C)$ is given by

\[T_P\Gr_{n,k}(\C) =\{[X,P]=(I-2P)X=-X(I-2P): X\in \u(n)\},
 \]
 
where $\u(n)$ is the space of skew-Hermitian matrices.
Next, from
Proposition 3.2 in \cite{BT2020} and \cite{BK2015} we know that  for a point \( P \in \Gr_{n,k}(\mathbb{C}) \) with a tangent vector  \( \overrightarrow{PQ} = [X, P] \), the geodesic \( \gamma(t) \) is given by
\[
\gamma(t) = \exp(tX) P \exp(-tX),
\]
where \( Q = \exp(X) P \exp(-X) \) and \( \exp \) denotes the matrix exponential function.
Moreover, 
for every $P\in \Gr_{n,k} (\C)$, there is a  $U_P\in U(n)$ such that $P=U_P P_0 U_P^{-1}$ by the Spectral Decomposition of Hermitian matrices, where
\[P_0:=\begin{pmatrix}I_k&0\\ 0& 0 \end{pmatrix} \in \Gr_{n,k}(\C).\]
With this spectral decomposition, we can further refine the description of the tangent space.
\begin{proposition} For every $P\in \Gr_{n,k} (\C)$, let  $U_P\in U(n)$ such that $P=U_P P_0 U_P^{-1}$. Then
\begin{eqnarray*}
T_P \Gr_{n,k}(\C) 
&=&\left\{  \left [U_P \begin{pmatrix} 0 & Y_1\\
-Y^*_1&0\end{pmatrix}U^{-1}_P, P\right]  : 
 Y_1\in \C_{k\times (n-k)}\right\}\\
 &=&[\Ad(U_P)\p_*, P],
\end{eqnarray*}

where 
$$\p_*: = \left\{\begin{pmatrix} 0 & Y_1\\
-Y^*_1&0\end{pmatrix}: Y_1^*\in \C_{(n-k)\times k}\right\}\subset \u(n).
$$
\end{proposition}
\begin{proof}
Suppose $P=U_P P_0 U_P^{-1}$. Given $X \in \u(n)$, express it as
$$
X = U_P \begin{pmatrix} X_1 & Y_1\\
-Y^*_1&Z_1\end{pmatrix}U^{-1}_P,\quad  X_1\in \u(n),\ Z_1 \in \u(n-k),\ Y_1\in \C_{k\times (n-k)}
$$
so 
\begin{equation} \label{[X,P]}
[X,P] = U_P \begin{pmatrix} 0 & -Y_1\\
-Y^*_1&0\end{pmatrix}U^{-1}_P = \left [ U_P \begin{pmatrix} 0& Y_1\\
-Y^*_1&0\end{pmatrix}U^{-1}_P,P \right]
\end{equation}
and  
\begin{equation} \label{XP+PX}
XP+PX = U_P \begin{pmatrix} 2X_1 & Y_1\\
-Y^*_1&0\end{pmatrix}U^{-1}_P.
\end{equation}
 Then
$$
T_P \Gr_{n,k}(\C) 
=\left\{  \left [U_P \begin{pmatrix} 0 & Y_1\\
-Y^*_1&0\end{pmatrix}U^{-1}_P, P\right]  : 
 Y_1\in \C_{k\times (n-k)}\right\}.
$$
So we can restrict the choices of $X\in \u(n)$ to a smaller set
$
\Ad(U_P)\p_*
\subset \u(n),
$
where 
$$\p_*: = \left\{\begin{pmatrix} 0 & Y_1\\
-Y^*_1&0\end{pmatrix}: Y_1^*\in \C_{(n-k)\times k}\right\}\subset \u(n).
$$
Thus
\[
T_P \Gr_{n,k}(\C) = \{{[X, P]}: X\in \Ad(U_P)\p_*\}.
\]
\end{proof}

\begin{proposition}\label{tangent vector} Let $P, Q\in \Gr_{n,k}(\C)$. Let $\gamma(t)$ be the geodesic joining $\gamma(0)=P$ and $\gamma(1) = Q$. Then there exists a  $\tilde{X}_{PQ}\in \p_*$ such that $[\Ad(U^{-1}_P)\tilde{X}_{PQ}, P]$ is the tangent vector to $\gamma$ at $P$, and there is $U_Q\in U(n)$ such that $U_Q P_o U_Q^{-1}=Q$ and 
\begin{equation}\label{X_PQ}
e^{\tilde{X}_{PQ}} = U_P^{-1}U_Q.
\end{equation}


\end{proposition}

\begin{proof}
 As $\ad P: \Ad(U_P) \p_*\to T_P \Gr_{n,k}(\C) $ is  surjective, there is a   $\tilde{X}_{PQ}\in  \p_*$ such that $\overrightarrow{PQ} = [\Ad(U^{-1}_P) \tilde{X}_{PQ}, P]$ is the tangent vector to the geodesic $\gamma(t)$ at $P$. 
 The map
\[
\pi: U(n) \to \Gr_{n,k} (\C), \quad U\to UP_0 U^{-1}.
\]
 is surjective and $\Gr_{n,k} (\C)=\pi(U(n))$ is  the orbit of $P_0$ under the adjoint action of 
$U(n)$. 
Then from the directional derivative of $\pi$ at $U_p$, we can find a geodesic $\hat{\gamma}(t):=U_Pe^{t \tilde{X}_{PQ}}$ in $U(n)$ with the tangent vector $U_P \tilde{X}_{PQ}$ at $U_P$
such that \[\pi(\hat{\gamma}(t))=\gamma(t)\]
and  
  $U_Q:=\hat{\gamma}(1)$.
When $t=1$, we have
$$
e^{\tilde{X}_{PQ}} = U_P^{-1}U_Q. 
$$
\end{proof}

To prove the triangle inequality, we use the following result from Thompson.
\begin{theorem}[Thompson, 1986]\label{Thompson} 
Let \( A, B \in \Cnn\) be skew-Hermitian matrices. Then there exist unitary matrices \( X, Y \in U(n) \) (depending on \( A \) and \( B \)) such that:
\begin{equation}\label{Thom}
e^{A} e^{B} = e^{X A X^{-1} + Y B Y^{-1}}.
\end{equation}
\end{theorem}

We would like to highlight that the relation in \eqref{Thom} also holds for quaternionic matrices in \( \H \), and this can be proved using the properties of the complex matrix representation of quaternions, as discussed in \cite{RL14}.
With all these tools in place, we can now state and prove the main result in the complex Grassmannians first.

\begin{theorem}\label{Triangular1}
Let $P, Q, R\in \Gr_{n,k}(\C)$. Then 
 \begin{equation}
\hat{d}(P, Q)\leq \hat{d}(Q, R)+\hat{d}(R, P).\end{equation}
\end{theorem}

\begin{proof}
By Proposition \ref{tangent vector},we can find $U_P, \hat{U}_P,$  $U_Q$ and $U_R$ such that
$$
e^{\tilde{X}_{PQ}} = U_P^{-1}U_Q, \quad  e^{\tilde{X}_{QR}}= U_Q^{-1} U_R, \quad e^{\tilde{X}_{RP}} = U_R^{-1} \hat{U}_P,
$$
where $U_P P_0 U^{-1}_P=\hat{U}_P P_0 \hat{U}^{-1}_P=P$, $U_Q P_0 U^{-1}_Q=Q$, and  $U_R P_0 U^{-1}_R=R$. 
Assume that $\hat{U}_P= U_P W$ and  
$$W:=\begin{bmatrix} W_k &\\ & W_{n-k} \end{bmatrix}\quad W_k\in U(k), W_{n-k}\in U(n-k).$$ 
Then
\begin{equation}\label{PQR}
e^{\tilde{X}_{PQ}} e^{\tilde{X}_{QR}}e^{\tilde{X}_{RP}}=W,
\end{equation}
where $\tilde{X}_{PQ}, \tilde{X}_{QR}, \tilde{X}_{PR}\in \p_*$.
Moreover, for any $X\in \C_{n\times n}$, $e^X=\sinh X+\cosh X$, $\sinh X=\frac{e^{X}-e^{-X}}{2}$ and $\cosh X=\frac{e^{X}+e^{-X}}{2}$. From \eqref{PQR}, we have
\[\begin{split}
&\sinh (\tilde{X}_{PQ}) \sinh (\tilde{X}_{QR}) \sinh (\tilde{X}_{RP})+
 \sinh (\tilde{X}_{PQ}) \cosh (\tilde{X}_{QR}) \cosh (\tilde{X}_{RP})\\
&+\cosh (\tilde{X}_{PQ}) \sinh (\tilde{X}_{QR}) \cosh (\tilde{X}_{RP})+
\cosh (\tilde{X}_{PQ}) \cosh (\tilde{X}_{QR}) \sinh (\tilde{X}_{RP})=0,
\end{split}\]
and 
\begin{equation}\label{XPQR}
e^{\tilde{X}_{PQ}} e^{\tilde{X}_{QR}}e^{\tilde{X}_{RP}}=e^{-\tilde{X}_{PQ}} e^{-\tilde{X}_{QR}}e^{-\tilde{X}_{RP}}.
\end{equation}

By \eqref{XPQR}, we have
\[e^{-2\tilde{X}_{PQ}}=e^{\tilde{X}_{QR}} e^{\tilde{X}_{RP}} e^{\tilde{X}_{RP}}e^{\tilde{X}_{QR}}.\]
Apply Theorem \ref{Thompson} to get two unitary matrices $M_1$ and $M_2$ such that
 \[e^{-2\tilde{X}_{PQ}}=e^{2M},
 \]
 
where $M=M_1\tilde{X}_{QR} M^*_1+M_2 \tilde{X}_{RP}M^*_2$.
When we only consider the case that all eigenvalues of $\tilde{X}_{PQ}$, $\tilde{X}_{QR}$, and $\tilde{X}_{RP}$ are all in 
$[-\pi, \pi]$, the corresonding distance will be the shortest distances.
The relation between eigenvalues of $\tilde{X}_{PQ}$ and $M$ is as following (make sure the eigenvalue of $2\tilde{X}_{PQ}$ is in $[-2\pi, 2\pi]$)
\[
\lambda_i(\tilde{X}_{PQ})=\begin{cases}
\lambda_i(M) +2\pi&   \mbox{if} \quad \lambda_i(M)< -\pi\\
\lambda_i(M)&  \mbox{if} \quad -\pi< \lambda_i(M)< \pi\\
\lambda_i(M) -2\pi&   \mbox{if} \quad \pi<\lambda_i(M)  \\
\end{cases}.
\]
Obviously, $\sigma(\tilde{X}_{PQ}) \prec_w \sigma(M)$.
So \[\|\tilde{X}_{PQ}\|_F \leq \|M\|_F=\|M_1\tilde{X}_{QR} M^*_1+M_2 \tilde{X}_{RP}M^*_2 \|_F\leq \|\tilde{X}_{QR}\|_F+\|\tilde{X}_{RP}\|_F,
\]
that is, 
\[\hat{d}(P, Q)\leq \hat{d}(Q, R)+\hat{d}(R, P).\]
 
\end{proof}

 
Finally, by Proposition \ref{Complex}, we can extend this result naturally to quaternionic Grassmannians:

\begin{corollary}\label{Triangular}
Let $P, Q, R\in \Gr_{n,k }(\H)$. Then 
 
\begin{equation}\hat{d}(P, Q)\leq \hat{d}(Q, R)+\hat{d}(R, P).\end{equation}
\end{corollary}

This completes the proof that the distance formula \eqref{dF} is a valid metric on \( \Gr_{n,k }(\H) \), satisfying all necessary conditions.

\begin{remark}
The results established in our previous work 
\cite{TW2024} can be extended to the quaternionic Grassmannian within certain locally convex ball. This extension follows from the properties of the complex representation of quaternionic matrices.
\end{remark}

\section{ Color Image Set Recognition based on Quaternionic Grassmannian}\label{sec4}

\subsection{Grassmannian Representation of a Color Image Set}
Each \( n \times m \) color digital image is represented as an \( n \times m \) quaternionic matrix, where each element in the matrix is a pure quaternion of the form \( r i + g j + b k \), where \( r, g, b \) are the red, green, and blue channel values of the corresponding pixel.

Consider
\[
Z=\begin{bmatrix} r_{11} i + g_{11} j + b_{11} k & \cdots & r_{1m} i + g_{1m} j + b_{1m} k \\ \vdots & \ddots & \vdots \\ r_{n1} i + g_{n1} j + b_{n1} k & \cdots & r_{nm} i + g_{nm} j + b_{nm} k \end{bmatrix} \in \H^{n \times m}.  
\]
The quaternionic matrix  $Z$ encodes the color image compactly by storing the RGB channels together.
The image matrix is partitioned into columns \( z_i \) and stacked into a long column vector of length \( nm \):

\[
[z_1| z_2|\cdots |z_m] \to  \z =\begin{bmatrix} z_1\\ z_2\\ \vdots\\ z_m \end{bmatrix}\in \H_{nm\times 1}.
\]

This column vector \( \z  \) provides a compact representation of the color image as a single quaternionic vector.

Then, consider a set \( \A = \{A_1, A_2, \dots, A_p\} \) of \( p \) color images with the same 
size $t\times m$. Each image can be represented by a quaternionic vector:
\[
 \fa_i \in \H_{tm \times 1}, \quad i = 1, 2, \dots, p.
\]
The image set can thus be written as a set of \( p \)  quaternionic vectors:
\[
\{ \fa_1, \fa_2, \dots, \fa_p \}.
\]
 
To reduce the dimensionality, we apply Quaternion Principal Component Analysis (QPCA) \cite{le2003} to transform the original $q$ quaternionic vectors into a lower-dimensional space. By retaining the top k components, we obtain a reduced set of representative vectors:
\[
\{ \hat{\fa}_1, \hat{\fa}_2, \dots, \hat{\fa}_k \}.
\]
These vectors span a subspace in the quaternionic vector space. To process the image set, we orthonormalize these vectors using the modified Gram-Schmidt process of Theorem 4.3 in \cite{F2003}  to obtain an orthonormal set:
\[
\{ \a_1, \a_2, \dots, \a_k \}.
\]
 
The resulting orthonormal vectors form the columns of a matrix \( X \in \H_{tm \times k} \):
\[
X = [  \a_1 | \a_2 | \dots | \a_k ].
\]

Finally, the quaternionic Grassmannian element representing the color image set is given by:
\[
A = XX^* \in \Gr_{tm,k}(\H)
\]
where  \( A \) is a projection matrix in the quaternionic Grassmannian \( \Gr_{tm,k}(\H) \), and it compactly encodes the color image set \( \A \).

Algorithm 1 outlines the procedure for representing a set of color images using quaternionic Grassmannians. Each image is transformed into a quaternionic column vector, and an orthonormal set is derived for efficient representation. The framework is illustrated in the figure \ref{fig1}.

\begin{algorithm}\label{Alg1}
    \caption{Quaternionic Grassmannian Representation of a Color Image Set}
    \begin{algorithmic}[1]
        \State \textbf{Input:} A color image set \( \A = \{A_1, A_2, \dots, A_p\} \) of \( p \)  color images with the same size $t\times m$ 
               \State Convert each color image into a quaternionic column vector:
        \For{each image \( A_i \)}
            \State Partition \( A_i \) into columns and stack them into a quaternionic vector \( \fa_i \in \mathbb{H}_{tm} \).
        \EndFor
        \State 
        Reduce \( \{ \fa_1, \fa_2, \dots, \fa_p \}\)     
 as  \( \{\hat{\fa}_1, \hat{\fa}_2, \dots, \hat{\fa}_k \}\)  by QSVD.
        \State 
        Orthonormalize the quaternionic vectors \( \{\hat{\fa}_1, \hat{\fa}_2, \dots, \hat{\fa}_k \} \) using the modified Gram-Schmidt process  \cite{F2003}  for quaternions to obtain an orthonormal set $ \{ \a_1, \a_2, \dots, \a_k \}$. 
        \State Form the matrix \( X \) with the orthonormal vectors as columns:
        \[
        X = [   \a_1 | \a_2 | \dots | \a_k ]
        \]
        \State Compute the Grassmannian representation \( A \) of the image set:
        \[
       A = XX^*
        \]
        \State \textbf{Output:} The Grassmannian element \( A \), representing the color image set $\A$.
    \end{algorithmic}
\end{algorithm}

\begin{figure}[h] 
   \centering
   \includegraphics[width=6.3in]{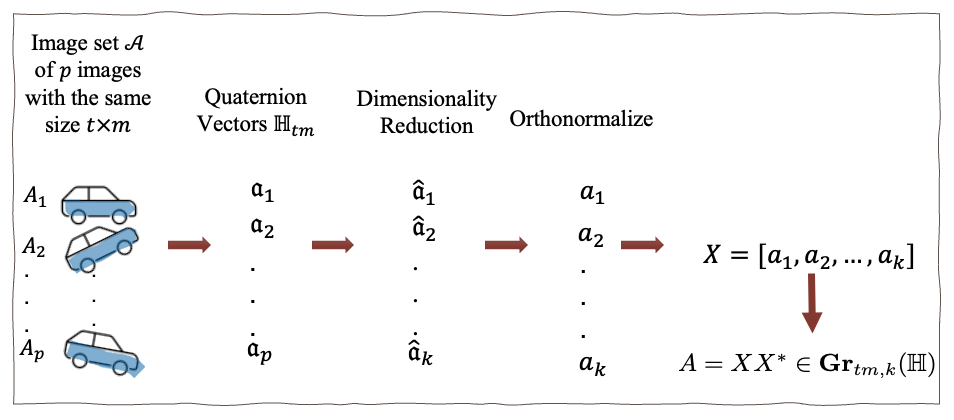} 
   \caption{Quaternionic Grassmannian representation of a color image set.}
   \label{fig1}
\end{figure}

\subsection{A New Framework for Color Image Set Recognition Based on Quaternionic Grassmannians}

Based on the Grassmannian representation of image sets and the distance formula, we propose a new framework for color image set recognition using quaternionic Grassmannians. As a straightforward example, we consider three image sets belonging to two distinct classes. The recognition process for these three image sets is described in Algorithm 2, and the framework is visually illustrated in Figure \ref{fig2}.

\begin{algorithm}
    \caption{Framework for Color Image Set Recognition Using Quaternionic Grassmannians (Three Image Sets)}
    \begin{algorithmic}[1]
        \State \textbf{Input:} Three color image sets \( \mathcal{A}, \mathcal{B}, \mathcal{C} \), each containing \( p \) color images of the same size \( n \times m \).
        \State Represent the three color image sets in \( \Gr_{n, k}(\H) \) (the Grassmannian space of quaternionic matrices) using Algorithm 1, obtaining representations \( A, B, C \).
        \State Compute the distances \( \hat{d}(A, B), \hat{d}(A, C), \) and \( \hat{d}(B, C) \).
        \State \textbf{Output:} The shortest distance identifies one class, and the remaining set belongs to the second class.
    \end{algorithmic}
\end{algorithm}

\begin{figure}[h]
   \centering
   \includegraphics[width=6in]{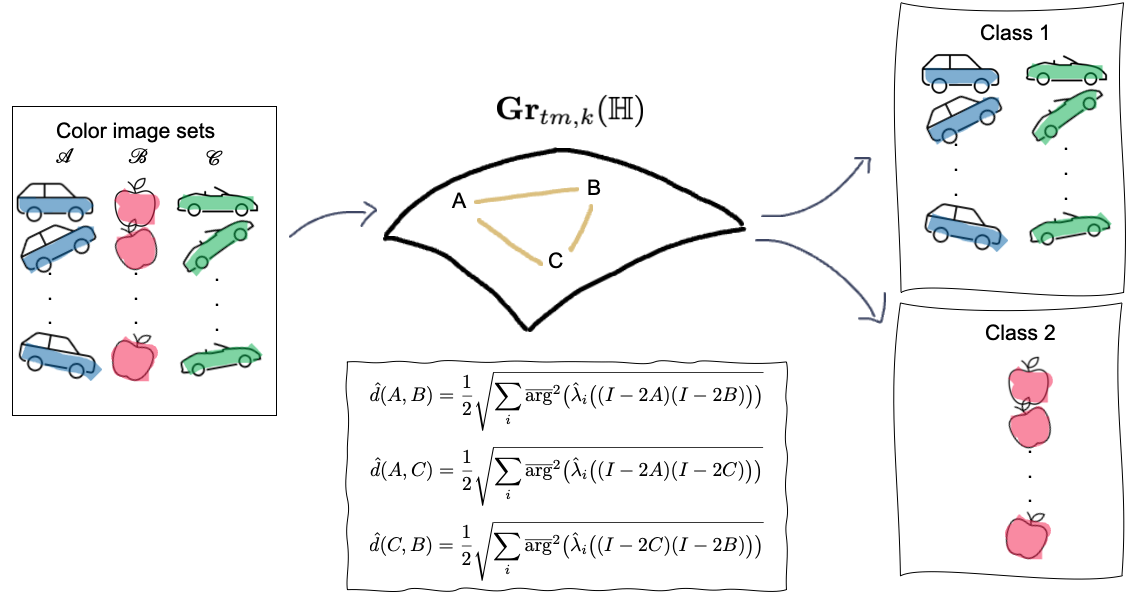}
   \caption{Framework for color image set recognition using quaternionic Grassmannians (three image sets)}
   \label{fig2}
\end{figure}

This framework establishes a fundamental structure for performing color image set recognition based on quaternionic Grassmannians. It provides a theoretical basis for developing more advanced algorithms and methods in future work, enabling more robust and efficient recognition techniques.

\section{Numerical Example}\label{sec5}

In this section, we evaluate the performance of our proposed quaternionic Grassmannian framework on two image set recognition tasks: the ETH-80 dataset \cite{LS2003} and a Highway Traffic dataset \cite{chan2005}. The MATLAB implementation used for all experiments in this paper is available at:  
\url{https://github.com/XiangXiangJY/QuaternionGrassmannian}.

\subsection{ETH-80 Dataset Evaluation}\label{sec5-eth80}

We first test our new framework using the ETH-80 dataset \cite{LS2003}. The ETH-80 dataset contains images from eight categories, including apples, pears, and cars. Each category has 10 objects with 41 views per object, resulting in a total of 3280 images. The following figure illustrates the ETH-80 dataset.

\begin{figure}[h]
   \centering
   \includegraphics[width=5in]{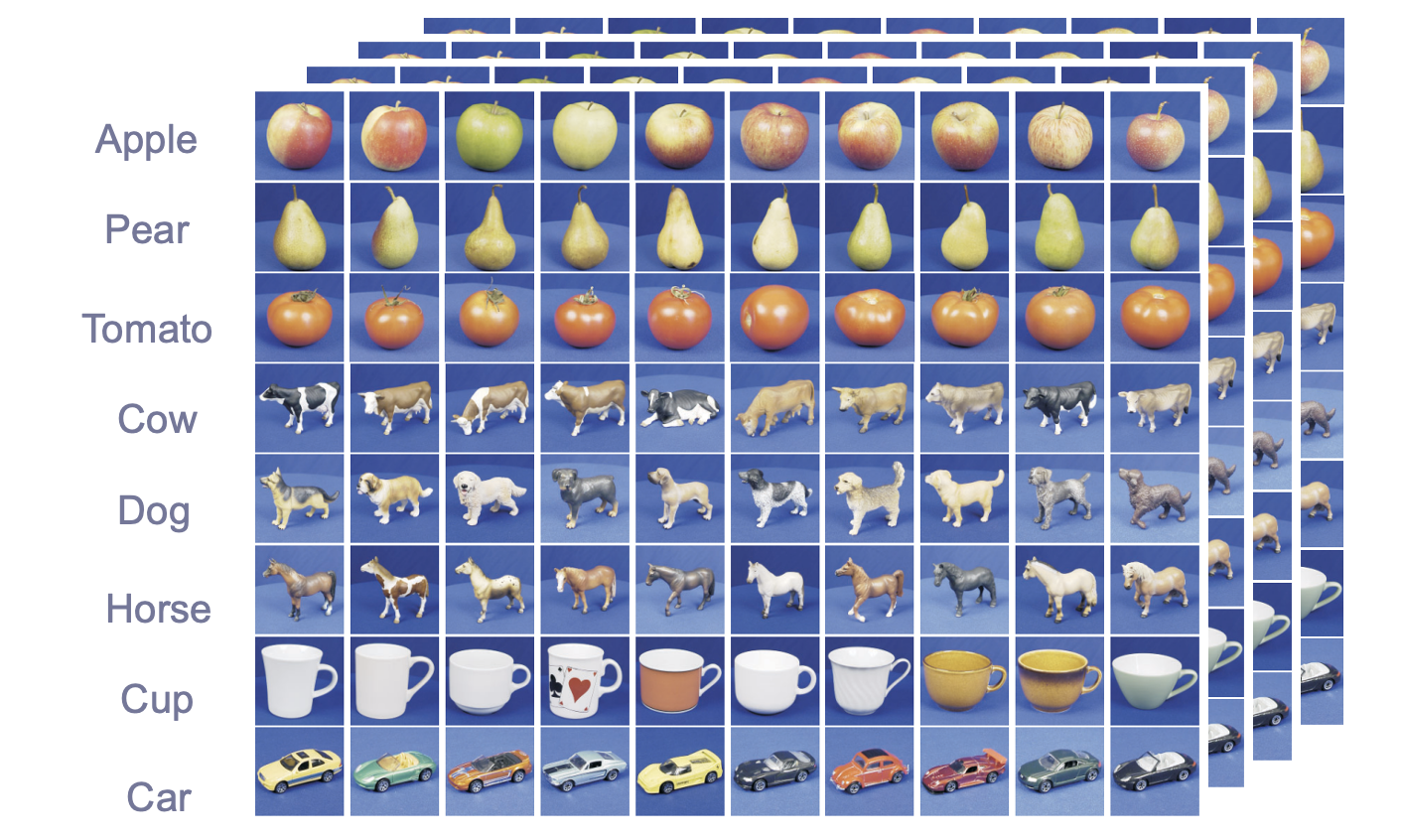}
   \caption{The eight categories  of the ETH-80 dataset. Each category contains 10 objects with 41 views per object.}
   \label{fig3}
\end{figure}

For this evaluation, the dataset was divided into training and testing sets. Five object instances selected from each category were used for training, and the remaining five were used for testing, ensuring a balanced and fair evaluation. Each image was resized to $20 \times 20$, with $t = 20$, $m = 20$, and we set $k = 9$. Consequently, the quaternionic Grassmannian considered in this work is $\Gr_{400,9}(\H )$.

In Figure \ref{fig4}, we apply Multidimensional Scaling (MDS) \cite{borg2005modern} to visualize one random split of the training dataset along with a selected 10  testing points from our trials. The eight categories are represented as clusters formed by the training data points, while the red point indicates the testing sample. To classify the test sample, we computed its average distance to each training cluster within the quaternionic Grassmannian space. The test sample was then assigned to the category corresponding to the nearest cluster.
\begin{figure}[h]
   \centering
   \includegraphics[width=5.5in]{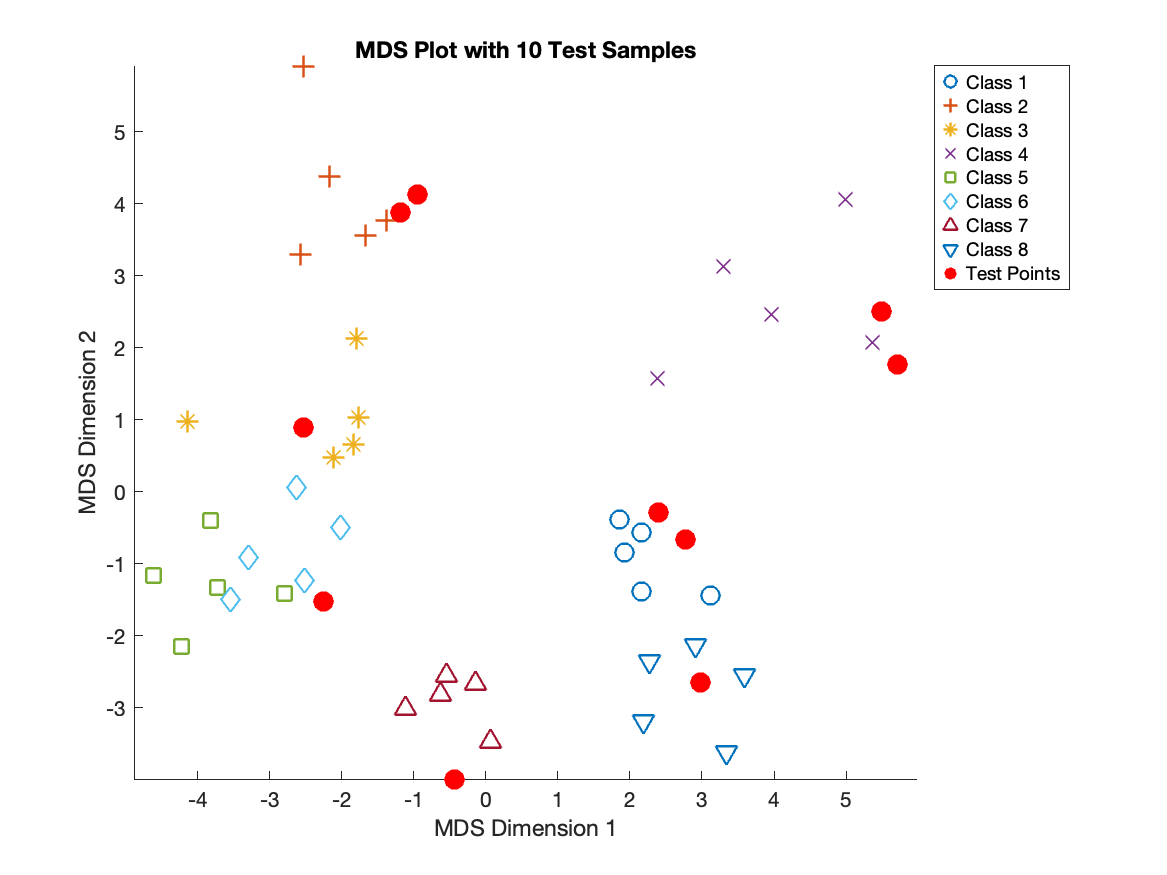}
   \caption{Visualization of the eight categories represented by training data points and the selected 10 testing data points.}
   \label{fig4}
\end{figure}

Using this framework, we computed the distances between points in the quaternionic Grassmannian to classify the test data. The process was repeated 10 times with different random training/testing splits in each repetition to ensure the reliability of the results, and the average recognition rates and standard deviations were calculated.  
Table \ref{tab:results} presents  the average recognition rates and standard deviations (\%) for various methods compared with our proposed approach.

\begin{table}[htbp]
   \centering
   \begin{tabular}{cc}
      \toprule
      \textbf{Method}   & \textbf{Recognition Rate (\%)} \\ 
      \midrule
     GDA \cite{hamm2008}   & 91.00 ± 2.13 \\ 
      GEDA \cite{Har2011}   & 92.50 ± 1.16 \\ 
      GDL \cite{Har2013}   & 93.50 ± 0.92 \\ 
      GiFME \cite{C2015}  & 74.50 ± 1.22 \\ 
      GGPLCR \cite{WEI2020} & 96.75 ± 1.30 \\ 
      \textbf{Our Method}  & \textbf{97.00 ± 1.97} \\ 
      \bottomrule
   \end{tabular}
   \caption{Average recognition rates and standard deviations (\%) on the ETH-80 dataset for various methods compared with our proposed framework.}
   \label{tab:results}
\end{table}

Our method achieved an average recognition rate of 97.00\% with a standard deviation of $\pm1.97\%$. While this demonstrates the framework’s ability to achieve high accuracy, the relatively large standard deviation indicates some variability in performance.

\begin{remark}
During our experiments, we observed some variations in recognition rates and standard deviations across multiple runs. Specifically, we conducted several independent rounds of evaluation, where each round involved 10 random trials. The average results for these rounds were as follows:  93.0 $\pm$ 3.5\%, 94.5 $\pm$ 3.07\%, 95.25 $\pm$ 2.75 \%, 96.75 $\pm$ 2.06 \%and 97.25 $\pm$ 2.75 \%. These fluctuations indicate that while our method achieves high recognition accuracy, its stability could be improved. This suggests that enhancing the robustness of our approach is an important direction for future work.
\end{remark}
\subsection{Highway Traffic Dataset Evaluation}\label{sec5-traffic}

To further validate the effectiveness and generalizability of our proposed framework, we conducted experiments on a Highway Traffic dataset \cite{chan2005}. This dataset consists of 254 color video sequences captured from highway surveillance cameras. Each video is categorized into one of three traffic conditions: \textbf{Heavy} (44 videos), \textbf{Medium} (45 videos), and \textbf{Light} (165 videos). Every video represents a short clip of traffic flow and contains between 42 and 52 frames. The videos are fully labeled according to their traffic condition in Figure \ref{figh1}.

\begin{figure}[h]
   \centering
   \includegraphics[width=4.5in]{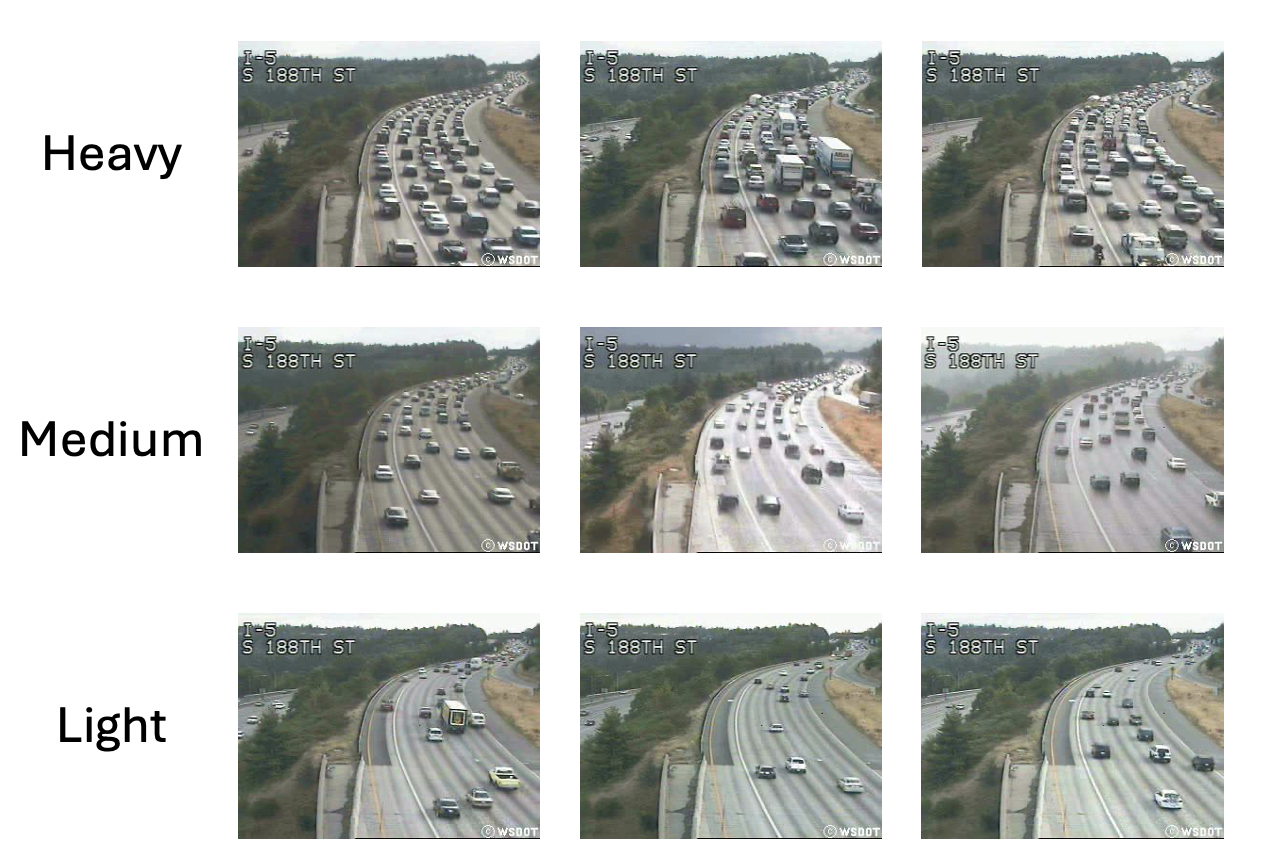}
   \caption{Sample frames from the Highway Traffic dataset. The dataset includes 44 Heavy, 45 Medium, and 165 Light traffic videos.}
   \label{figh1}
\end{figure}

For the experimental evaluation, we followed the protocol described in Wei et al.'s paper (2024) \cite{wei2024} to ensure a fair comparison. Each video was treated as an image set by extracting all its frames. Following the experimental setting in \cite{wei2024}, each frame was resized to $24 \times 24$, and we set $k = 9$. As a result, each image set corresponds to a point on the quaternionic Grassmannian $\Gr_{576,9}(\mathbb{H})$.

To assess classification performance, we randomly selected 192 video samples from the dataset for training, with the remaining samples used for testing. This random partitioning was repeated 10 times, each with a newly generated training/testing split. In each trial, the recognition rate was computed, and we report the final result as the average recognition rate and standard deviation over all 10 trials.
The MDS plot of the Highway Traffic dataset is shown in Figure \ref{figh7}, using the same methodology as in Figure  \ref{fig4}.

\begin{figure}[h]
   \centering
   \includegraphics[width=5.5in]{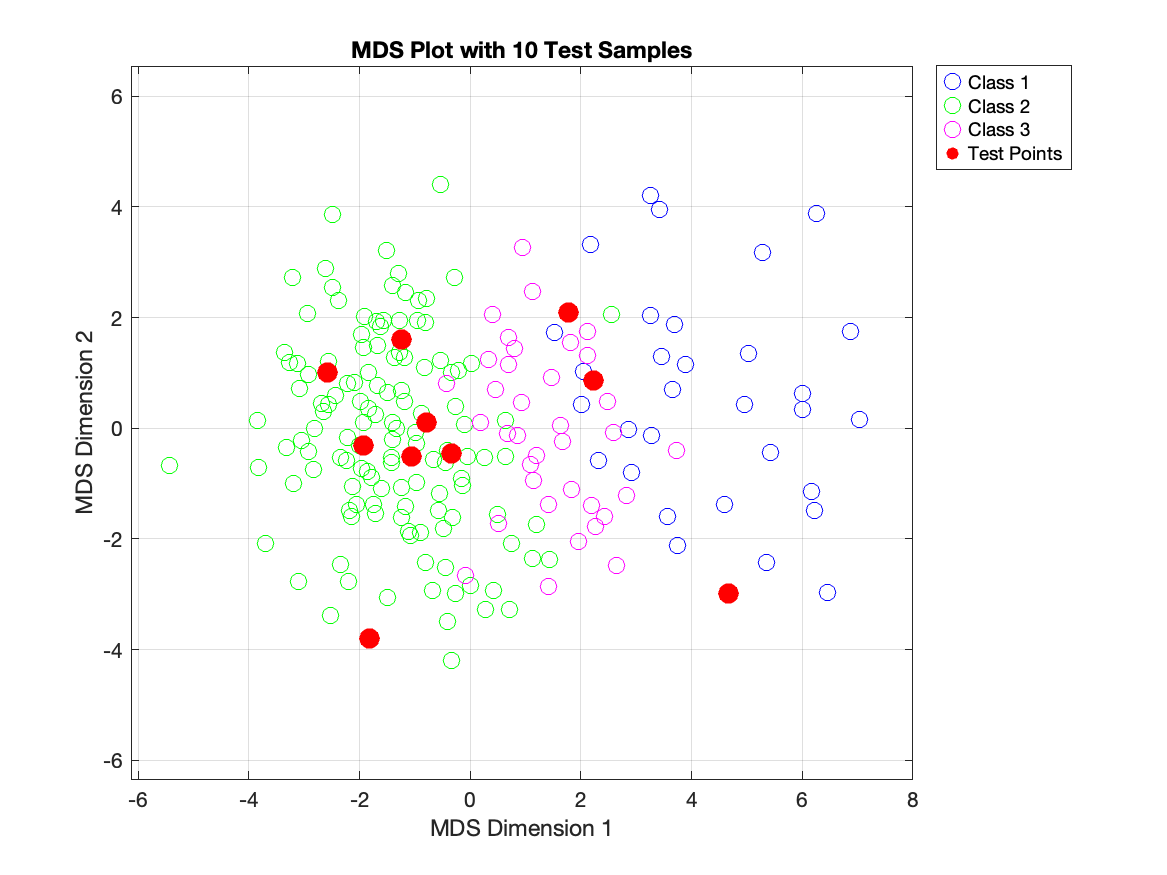}
   \caption{Visualization of the three categories represented by training data points and the selected 10 testing data points.}
   \label{figh7}
\end{figure}

Table \ref{tab:results_traffic} presents the recognition rates for various methods, including those reported in \cite{wei2024}, compared with our proposed approach.

\begin{table}[htbp]
   \centering
   \begin{tabular}{cc}
      \toprule
      \textbf{Method}   & \textbf{Recognition Rate (\%)} \\ 
      \midrule
        GNN  \cite{hamm2008} &70.00   \\ 
      GLPP \cite{kumar2019jumping} & 76.67 \\ 
       GDA \cite{hamm2008}   &75.69   \\ 
      GEDA \cite{Har2011}   & 78.96  \\ 
        GNPE\cite{wei2022neighborhood} & 75.66\\
        GALL \cite{wei2024}&  78.95\\
        F-GALL \cite{wei2024}& 79.04\\
       \textbf{Our Method}  & \textbf{88.55 ± 2.46} \\ 
      \bottomrule
   \end{tabular}
   \caption{Average recognition rates  (\%) on the Highway Traffic dataset for various methods compared with our proposed framework.}
   \label{tab:results_traffic}
\end{table}

Although the other methods do not report the standard deviations over their trials, it is evident that our proposed method achieves the highest average recognition rate. Specifically, our method achieved an average recognition accuracy of 88.55\% with a standard deviation of $\pm$ 2.46\%,   significantly outperforming all baseline approaches in terms of classification accuracy.

Despite this variability, we believe our framework shows great potential. Since advanced techniques such as deep learning or feature selection methods have not yet been incorporated, there is significant room for optimization and improvement. Future enhancements could further improve the framework’s performance and robustness.
\section{Conclusion}\label{sec6}

This work introduced significant contributions to the recognition of color image sets using quaternionic Grassmannians. We developed an explicit formula for computing the shortest path between points in quaternionic Grassmannians, offering a mathematically sound and efficient approach for distance calculations. Additionally, we proposed a novel framework for color image set recognition that effectively utilizes the structure of quaternionic Grassmannians to handle high-dimensional, multi-channel image data.

Despite these advancements, there is room for further improvement. 
\begin{itemize}
\item One limitation of our current method is that it relies on computing the standard eigenvalues of quaternionic unitary matrices of the form $(I - 2Q)(I - 2P)$. This step is time-consuming and slows down the process for large image sets. If we can develop faster ways to compute these eigenvalues, our method would work better for large-scale and real-time applications.
\item Another direction involves incorporating advanced techniques, such as deep learning and feature extraction, to enhance the robustness and accuracy of the framework. These additions could expand its applicability to more complex and diverse datasets.

\end{itemize}
By addressing these challenges, this framework has the potential to evolve into a powerful tool for a wide range of applications in computer vision and image processing, contributing to advancements in the field.


\end{document}